\documentclass[conference]{IEEEtran}
\IEEEoverridecommandlockouts

\usepackage{ifthen}

\newboolean{techreport}
\setboolean{techreport}{true}

\newboolean{cameraready}
\setboolean{cameraready}{true}

\usepackage{xr-hyper}
\usepackage[hypertexnames=false,hidelinks]{hyperref}
\usepackage{etoolbox}

\usepackage[utf8]{inputenc}
\usepackage[T1]{fontenc}
\usepackage[table,dvipsnames]{xcolor}
\usepackage{xparse,pgffor,xspace,glossaries,url}
\usepackage[setbb,setcolon]{kmath}
\usepackage{amsmath,amssymb,amsfonts,amsthm,mathrsfs,mathtools,bbm,stmaryrd}
\usepackage{pgfplots,forest}
\usepgfplotslibrary{groupplots}
\usetikzlibrary{matrix}
\pgfplotsset{compat=newest}
\usepackage{cleveref,csquotes,comment,lipsum,soul,adjustbox,booktabs}

\usepackage[numbers,sort&compress]{natbib}

\pgfplotscreateplotcyclelist{mylist}{
    smooth, ultra thick, mark=*, mark size=0.75pt, RoyalBlue\\
    smooth, ultra thick, mark=*, mark size=0.75pt, BlueGreen\\
    smooth, ultra thick, mark=*, mark size=0.75pt, LimeGreen\\
    smooth, ultra thick, mark=*, mark size=0.75pt, Dandelion\\
    smooth, ultra thick, mark=*, mark size=0.75pt, OrangeRed\\
}

\pgfplotsset{
 	legend image code/.code={
		\draw[mark repeat=2,mark phase=2] plot coordinates {
			(0cm,0cm)
			(0.125cm,0cm)
			(0.25cm,0cm)
		};
    }
}

\newtheorem{proposition}{Proposition}
\newtheorem{lemma}{Lemma}

\AfterEndEnvironment{definition}{\noindent\ignorespaces}
\AfterEndEnvironment{remark}{\noindent\ignorespaces}
\AfterEndEnvironment{proposition}{\noindent\ignorespaces}
\AfterEndEnvironment{lemma}{\noindent\ignorespaces}
\AfterEndEnvironment{theorem}{\noindent\ignorespaces}
\AfterEndEnvironment{corollary}{\noindent\ignorespaces}
\AfterEndEnvironment{example}{\noindent\ignorespaces}
\AfterEndEnvironment{proof}{\noindent\ignorespaces}
\Crefname{section}{Sec.}{Sec.}
\Crefname{proposition}{Prop.}{Prop.}
\Crefname{lemma}{Lem.}{Lem.}
\Crefname{corollary}{Cor.}{Cor.}
\Crefname{appendix}{App.}{App.}
\Crefname{figure}{Fig.}{Fig.}
\Crefname{equation}{Eq.}{Eq.}
\Crefname{algorithm}{Alg.}{Alg.}
\Crefname{theorem}{Thm.}{Thm.}
\Crefname{remark}{Rem.}{Rem.}

\def\0{\mathbf{0}}
\def\1{\mathbf{1}}

\def\ie{\emph{i.e.,}\xspace}
\def\eg{\emph{e.g.,}\xspace}

\ifthenelse{\boolean{cameraready}}{
	\newcommand{\remCE}[1]{}
	
	\newcommand{\supCE}[1]{}

	\newcommand{\remCH}[1]{}
	
	\newcommand{\supCH}[1]{}

	\newcommand{\remTG}[1]{}
	
	\newcommand{\supTG}[1]{}

	\newcommand{\remGM}[1]{}
	
	\newcommand{\supGM}[1]{}
}{
	\newcommand{\remCE}[1]{{\noindent\color{purple}{\scriptsize[CE: #1]}}}
	
	\newcommand{\supCE}[1]{{\noindent\color{purple}{\textst{#1}}}}

	\newcommand{\remCH}[1]{{\noindent\color{red}{\scriptsize[CH: #1]}}}
	
	\newcommand{\supCH}[1]{{\noindent\color{red}{\textst{#1}}}}

	\newcommand{\remTG}[1]{{\noindent\color{blue}{\scriptsize[TG: #1]}}}
	
	\newcommand{\supTG}[1]{{\noindent\color{blue}{\textst{#1}}}}

	\newcommand{\remGM}[1]{{\noindent\color{olive}{\scriptsize[GM: #1]}}}
	
	\newcommand{\supGM}[1]{{\noindent\color{olive}{\textst{#1}}}}
}

\newcommand{\regfunc}{g}
\newcommand{\sparsitylevel}{k}
\newcommand{\pdim}{n}
\newcommand{\ddim}{m}
\newcommand{\groundtruth}{\pv^{\dagger}}
\newcommand{\groundtruthi}{\pvi^{\dagger}}
\newcommand{\obs}{\mathbf{y}}
\newcommand{\dic}{\mathbf{A}}
\newcommand{\atom}{\mathbf{a}}
\newcommand{\bigM}{M}
\newcommand{\reg}{\lambda}

\newcommand{\UB}[1]{\bar{#1}}

\newcommand{\ppb}{\mathcal{P}}
\newcommand{\rpb}{\mathcal{R}}

\newcommand{\pobj}{p}
\newcommand{\robj}{r}

\newcommand{\minSet}{\opt{\mathcal{X}}}

\newcommand{\pfunc}{\mathrm{P}}
\newcommand{\rfunc}{\mathrm{R}}
\newcommand{\dfunc}{\mathrm{D}}
\newcommand{\indicator}{\eta}

\newcommand{\pvi}{x}
\newcommand{\dvi}{w}
\newcommand{\cvi}{v}
\newcommand{\pv}{\mathbf{\pvi}}
\newcommand{\dv}{\mathbf{\dvi}}

\newcommand{\bigLi}{l}
\newcommand{\bigUi}{u}
\newcommand{\bigL}{\mathbf{\bigLi}}
\newcommand{\bigU}{\mathbf{\bigUi}}
\newcommand{\bigLnew}{\bigL'}
\newcommand{\bigUnew}{\bigU'}

\newcommand{\idxentry}{i}       
\newcommand{\idxpeeled}{j}      

\newcommand{\nodeSymb}{\nu}

\newcommand{\node}[1]{#1^{\nodeSymb}}
\newcommand{\setidx}{\nodeSymb}
\newcommand{\symbzero}{0}
\newcommand{\symbone}{1}
\newcommand{\symbnone}{\bullet}
\newcommand{\setzero}{\setidx_\symbzero}
\newcommand{\setone}{\setidx_\symbone}
\newcommand{\setnone}{\setidx_\symbnone}
\newcommand{\subzero}[1]{#1_{\setzero}}
\newcommand{\subone}[1]{#1_{\setone}}

\newcommand{\pivot}[1]{\mu_{#1}}

\newcommand{\peelthres}{\alpha}
\newcommand{\peeloffset}{\psi}

\newcommand{\stdnnz}{\sigma}
\newcommand{\corrmatrix}{\mathbf{K}}
\newcommand{\corrmatrixel}{K}
\newcommand{\corrparam}{\rho}

\newcommand{\noise}{\mathbf{n}} 
\newcommand{\bigMfactor}{\gamma}

\newcommand{\abs}[1]{|#1|}

\newcommand{\conj}[1]{#1^\star}
\newcommand{\card}[1]{|#1|}

\newcommand{\intervint}[2]{\llbracket#1,#2\rrbracket}
\newcommand{\norm}[2]{\|#1\|_#2}
\newcommand{\opt}[1]{#1^{\star}}
\newcommand{\pospart}[1]{[#1]_+}

\newcommand{\sign}[1]{\mathrm{sign}({#1})}

\newcommand{\transp}[1]{#1^{\mathrm{T}}}

\newacronym{bnb}{BnB}{Branch-and-Bound}

\begin{document}

\title{%
    Safe Peeling for $\ell_0$-Regularized Least-Squares
    \ifthenelse{\boolean{techreport}}{\\ with supplementary material}{}
}

\author{
    \IEEEauthorblockN{
        Théo Guyard\IEEEauthorrefmark{1}\IEEEauthorrefmark{2},
        Gilles Monnoyer\IEEEauthorrefmark{3},
        Cl\'ement Elvira\IEEEauthorrefmark{4} 
        and
        C\'edric Herzet\IEEEauthorrefmark{1}
    } \\
    \IEEEauthorblockA{\IEEEauthorrefmark{1}Inria, Centre de l'Université de Rennes, France}%
    \IEEEauthorblockA{\IEEEauthorrefmark{2}IRMAR UMR CNRS 6625, INSA Rennes, France}%
    \IEEEauthorblockA{\IEEEauthorrefmark{3}ELEN, ICTEAM, UCLouvain, Belgium}%
    \IEEEauthorblockA{\IEEEauthorrefmark{4}IETR UMR CNRS 6164, CentraleSupelec Rennes Campus, France}%
    \thanks{
        Gilles Monnoyer is funded by the Belgian FNRS.
        The research presented in this paper is reproducible.
        The code associated to our numerical experiments is available at {\protect\url{https://github.com/TheoGuyard/BnbPeeling.jl}}.
    }
}

\maketitle

\begin{abstract}
    We introduce a new methodology dubbed \emph{``safe peeling''} to accelerate the resolution of \(\ell_0\)-regularized least-squares problems via a \gls{bnb} algorithm.
    Our procedure enables to tighten the convex relaxation considered at each node of the \gls{bnb} decision tree and therefore potentially allows for more aggressive pruning. 
    Numerical simulations show that our proposed methodology leads to significant gains in terms of number of nodes explored and overall solving time.
\end{abstract}



\section{Introduction}

\glsreset{bnb}

This paper focuses on the resolution of the so-called ``\(\ell_0\)-regularized least-squares'' problem given by
\begin{align}
	\stepcounter{equation}
	\tag{\theequation-\(\ppb\)}
	\label{prob:prob-base} 
	\opt{\pobj} = \min_{\pv \in \kR^{\pdim}} \ \pfunc(\pv) \triangleq \tfrac{1}{2}\kvvbar{\obs - \dic\pv}_2^2 + \reg \norm{\pv}{0}
\end{align}
where \(\obs \in \kR^{\ddim}\) and \(\dic \in \kR^{\ddim\times\pdim}\) are input data, \(\reg>0\) is a regularization parameter and \(\kvvbar{\cdot}_0\) denotes the $\ell_0$-pseudonorm which counts the number of non-zero elements in its argument.  

Solving~\eqref{prob:prob-base} is of paramount interest in many scientific fields such as statistics, machine learning or inverse problems~\cite{candes2005decoding,donoho2005stable,tropp2010computational}. 
Unfortunately, this problem also turns out to be NP-hard~\cite[Th.~1]{Chen2012}. 
Hence, the last decades have seen a flurry of contributions proposing tractable procedures able to recover approximate solutions of~\eqref{prob:prob-base}. 
Canonical examples include greedy algorithms or methodologies based on convex relaxations, see~{\cite[Ch.~3]{foucart2013invitation}.
Although these procedures successfully recover the actual solutions of \eqref{prob:prob-base} in ``easy'' setups, they usually fall short for more challenging instances of the problem.  
This observation, combined with some recent advances in integer optimization and hardware performance, has revived the interest in methods solving~\eqref{prob:prob-base} exactly.
A standard approach is to use a \gls{bnb} algorithm that solves~\eqref{prob:prob-base}, see~\cite{cplex2009v12,bourguignon2015exact,Bertsimas2016,bertsimas2021unified,hazimeh2021sparse,guyard2022node}.

In this paper, we propose a new strategy, dubbed \textit{``safe peeling''}, to accelerate the exact resolution of~\eqref{prob:prob-base}.
In a nutshell, our contribution is a computationally simple test applied at each node of the \gls{bnb} decision tree to identify some intervals of \(\kR^\pdim\) which cannot contain a solution of~\eqref{prob:prob-base}. 
This information allows to construct tighter convex relaxations and more aggressive pruning of the nodes of the decision tree. 
Our numerical experiments show that the proposed method leads to a significant reduction of the solving time as compared to state-of-the-art concurrent methods. 
The name \textit{``safe peeling''} comes from the fact that the proposed method enables to reduce (or in more figurative terms, ``to peel'') the feasible set of the problem at each node of the decision tree while safely preserving the correctness of the \gls{bnb} procedure.

The rest of the paper is organized as follows.
\Cref{sec:bnb} describes the main ingredients of \gls{bnb} methods.
Our peeling strategy is presented in \Cref{sec:peeling} and its performance is illustrated in \Cref{sec:results}.
\ifthenelse{\boolean{techreport}}{
	Proofs of the results presented in the following are postponed to the appendix.
}{
	All the proofs are postponed to the technical report accompanying this paper~\cite{peeling2023technicalreport}.
} 


\section{Notations}

We use the following notations.
\(\0\) and \(\1\) denote the all-zero and all-one vectors.
The \(\idxentry\)-th column of a matrix \(\dic\) is denoted \(\atom_\idxentry\) and the \(\idxentry\)-th entry of a vector \(\pv\) is denoted \(\pvi_{\idxentry}\). 
The superscript $\ktranspose{}$ refers to transposition. 
Any vectorial relation has to be understood component-wise, \eg $\pv \in [\bigL,\bigU]$ means $\pvi_\idxentry \in [\bigLi_\idxentry,\bigUi_\idxentry], \forall \idxentry$. 
Moreover, $\indicator(\cdot)$ denotes the indicator function which equals to $0$ if the condition in argument is fulfilled and to $+\infty$ otherwise,
\(\pospart{\pvi}=\max(\pvi,0)\) refers to the positive-part function and
\(\card{\cdot}\) denotes the cardinality of a set. 
Finally, $\intervint{1}{\pdim}$ with \(\pdim\in\kN*\) is a short-hand notation for the set $\{1,\ldots,\pdim\}$. 


\section{Principles of BnB Methods}
\label{sec:bnb}

In this section, we recall the main principles of \gls{bnb} procedures. 
Due to space limitation, we only review the elements of interest to introduce the proposed peeling method. 
We refer the reader to \cite[Ch.~7]{Wolsey:1998op} for an in-depth treatment of the subject. 

\subsection{Pruning}

The crux of \gls{bnb} methods consists in identifying and discarding some subsets of $\kR^{\pdim}$ which do not contain a minimizer of \eqref{prob:prob-base}.  
To do so, one constructs a decision tree in which each node corresponds to a particular subset of $\kR^{\pdim}$. 
In our context, a tree node is identified by two disjoint subsets of $\intervint{1}{\pdim}$, say $\setzero$ and $\setone$.  
The goal at node $\nodeSymb \triangleq (\setzero,\setone)$ is to detect whether a solution of \eqref{prob:prob-base} can be attained within
\begin{align}
	\node{\mathcal{X}}
	\triangleq \kset{\pv\in\kR^\pdim}{\subzero{\pv} = \0, \ \subone{\pv} \neq \0}
	,
\end{align} 
where \(\pv_{\nodeSymb_k}\) denotes the restriction of $\pv$ to its elements in \(\nodeSymb_k\). 
In particular, let $\minSet$ be the non-empty set of minimizers of \eqref{prob:prob-base}.
Then, if some upper bound $\UB{\pobj}$ on the optimal value \(\opt{\pobj}\) is known and if we let
\begin{equation}
	\label{eq:node_problem}
	\node{\pobj} \triangleq \inf_{\pv \in \kR^{\pdim}} \node{\pfunc}(\pv)
\end{equation}
with $\node{\pfunc}(\pv) \triangleq \pfunc(\pv) + \indicator(\pv \in \node{\mathcal{X}})$, we obtain the implication
\begin{align}
	\label{eq:pruning condition}
	\node{\pobj} > \UB{\pobj} \implies  \node{\mathcal{X}}\cap\minSet=\emptyset
	.
\end{align}
In words, if the left-hand side of \eqref{eq:pruning condition} is satisfied, $\node{\mathcal{X}}$ does not contain any solution of \eqref{prob:prob-base} and can therefore be discarded from the search space of the optimization problem. 
This operation is usually referred to as ``\textit{pruning}''. 

\subsection{Bounding and relaxing}

Making a pruning decision at node $\nodeSymb$ requires the knowledge of $\UB{\pobj}$ and $\node{\pobj}$. 
On the one hand, finding $\UB{\pobj}$ is an easy task since the value of the objective function in \eqref{prob:prob-base} at any feasible point constitutes an upper bound on $\opt{\pobj}$.
On the other hand, evaluating $\node{\pobj}$ is NP-hard. 
This issue can nevertheless be circumvented by finding a tractable lower bound $\node{\robj}$ on $\node{\pobj}$ and relaxing \eqref{eq:pruning condition} as
\begin{align}
	\label{eq:relaxed_pruning_condition}
	\node{\robj} > \UB{\pobj} \implies  \node{\mathcal{X}}\cap\minSet=\emptyset
	.
\end{align}
One ubiquitous approach in the literature~\cite{ben2022global,bertsimas2021unified,bourguignon2015exact} to find such a lower bound consists in:
\begin{itemize}
	\item[\textit{i)}] 
	Adding an extra term ``$\indicator(\pv \in [\bigL,\bigU])$'' to the cost function of \eqref{prob:prob-base}, for some \textit{well-chosen} bounds $\bigL\in\kR-^\pdim$ and $\bigU\in\kR+^\pdim$.\footnote{This additional constraint usually takes the form ``$-\bigM \leq \pvi_{\idxentry} \leq \bigM, \ \forall \idxentry$'' with $\bigM>0$ and is known as \textit{``Big-M''} constraint, see~\cite[Sec.~3]{bourguignon2015exact}}	
	In particular, the new constraint ``$\pv \in [\bigL,\bigU]$'' must lead to a problem fully equivalent to \eqref{prob:prob-base}, that is
	\begin{align}
		\label{eq:condition_equivalence}
		\forall \opt{\pv}\in\minSet:\ \opt{\pv}\in [\bigL,\bigU]. 
	\end{align}
	\item[\textit{ii)}] 
	Exploiting the convex relaxation of the function $\kvvbar{\cdot}_0$ on the bounded set
	$\node{\mathcal{X}}\cap [\bigL,\bigU]$, given by
	\begin{align} 
		\label{eq:convex_relaxation_l0norm}
		\|\pv\|_0
		\geq 
		|\setone| + \sum_{\idxentry\in \setnone} 
		\frac{\pospart{\pvi_{\idxentry}}}{\bigUi_{\idxentry}} - \frac{\pospart{-\pvi_{\idxentry}}}{\bigLi_{\idxentry}}
		,
	\end{align}
	with $\setnone \triangleq \intervint{1}{\pdim} \setminus(\setzero\cup\setone)$ and the convention ``\(0/0=0\)''.
\end{itemize}
On the one hand, item \textit{i)} implies that the pruning test \eqref{eq:pruning condition} involves the following quantity (rather than $\node{\pobj}$):
\begin{align}
	\label{eq:node_problem_constr}
	\stepcounter{equation}
	\tag{\theequation-\(\node{\ppb}\)}
	\node{\pobj}(\bigL,\bigU) 
		= \inf_{\pv \in \kR^{\pdim}} \ \node{\pfunc}(\pv;\bigL,\bigU)
\end{align}
where $\node{\pfunc}(\pv;\bigL,\bigU) \triangleq \node{\pfunc}(\pv) + \indicator(\pv \in [\bigL,\bigU])$.
On the other hand, a lower bound $\node{\robj}(\bigL,\bigU)$ on $\node{\pobj}(\bigL,\bigU)$ can be obtained by using \eqref{eq:convex_relaxation_l0norm}  and solving
\begin{align}
	\label{eq:node_relaxed_problem}
	\stepcounter{equation}
	\tag{\theequation-\(\node{\rpb}\)}
	\node{\robj}(\bigL,\bigU)  = \min_{\pv\in\kR^\pdim} \node{\rfunc}(\pv;\bigL,\bigU)
\end{align}
where 
\begin{multline*} 
	\node{\rfunc}(\pv;\bigL,\bigU) 
	\triangleq 
	\tfrac{1}{2}\kvvbar{\obs - \dic\pv}_2^2 
	+ \reg \sum_{\idxentry \in \setnone} \tfrac{\pospart{\pvi_{\idxentry}}}{\bigUi_\idxentry}
	- \tfrac{\pospart{-\pvi_{\idxentry}}}{\bigLi_\idxentry}
	\\
	+ \reg \card{\setone} 
	+ \indicator(\subzero{\pv} = \0)
	+ \indicator(\pv \in [\bigL,\bigU]) 
	.
\end{multline*}
We note that \eqref{eq:node_relaxed_problem} is a convex  problem and can be solved efficiently to good accuracy via numerous polynomial-time numerical procedures, see \eg~\cite[Ch.~10]{Beck2017aa}.

In practice, the choice of $\bigL$ and $\bigU$ must respect two conflicting imperatives. 
First, the new constraint ``$\pv \in [\bigL,\bigU]$'' should not modify the solution of our target problem~\eqref{prob:prob-base} and condition \eqref{eq:condition_equivalence} must therefore be verified. 
Since $\minSet$ is obviously not accessible beforehand, 
this suggests that the entries of $\bigL$ and $\bigU$ should be chosen  ``large-enough'' in absolute values.\footnote{
	Some heuristics are commonly used in the literature to select proper values of the bounds, see \cite[Sec.~V.B]{bourguignon2015exact}, \cite[Sec.~5.1]{guyard2022node} or \cite[Sec.~4]{mhenni2020sparse}. 
} 
Second, the tightness of $\node{\robj}(\bigL,\bigU)$ with respect to $\node{\pobj}(\bigL,\bigU)$ degrades with the spread of the set \([\bigL,\bigU]\).\footnote{
	This impairment pertains to a large class of mixed-integer problems and is well known in the literature, see \textit{e.g.},~\cite{camm1990cutting}.
} 
In particular, the right-hand side of \eqref{eq:convex_relaxation_l0norm} tends to $|\setone|$  when $\bigL \ll \pv$ and $\pv \ll \bigU$.  
Therefore, setting the entries of $\bigL$ and $\bigU$ with too large absolute values is likely to degrade the effectiveness of the relaxed pruning decision~\eqref{eq:relaxed_pruning_condition}. 

In the next section, we propose a solution to address this problem by deriving a methodology which locally tightens the constraint $\pv \in [\bigL,\bigU]$ at each node of the decision tree while preserving the correctness of the \gls{bnb} procedure. 


\section{Peeling}
\label{sec:peeling}

In this section, we introduce our proposed peeling procedure.
As an initial assumption, we suppose that some interval $[\bigL,\bigU]$ verifying condition \eqref{eq:condition_equivalence} is known. 
This assumption will be relaxed later on in \Cref{sec:propagation}. 

Our goal is to find a new interval $[\bigLnew,\bigUnew]$ such that 
\begin{subequations}
	\begin{align}
		\forall \pv\in [\bigL,\bigU]\setminus[\bigLnew,\bigUnew]&: \ \node{\pfunc}(\pv;\bigL,\bigU)>\UB{\pobj} 
		\label{eq:peeling:constraint}
		\\
		[\bigLnew,\bigUnew]&\subseteq [\bigL,\bigU]
		\label{eq:peeling:constraint_mnew}
		.
	\end{align}	
\end{subequations}
These requirements imply that the pruning decision \eqref{eq:pruning condition} made at node $\nodeSymb$ remains unchanged when replacing  constraint ``$\pv\in [\bigL,\bigU]$'' by ``$\pv\in [\bigLnew,\bigUnew]$'' in \eqref{eq:node_problem_constr}. More specifically, the following result holds: 
\begin{lemma}
	\label{lemma: consequence properties peeled interval}
 	Assume $[\bigL,\bigU]$ and $[\bigLnew,\bigUnew]$ verify \eqref{eq:peeling:constraint}-\eqref{eq:peeling:constraint_mnew}, then 
 	\begin{align}\label{eq:peeling:equivalence}
 		\node{\pobj}(\bigLnew,\bigUnew)> \UB{\pobj}	
		&\iff
		\node{\pobj}(\bigL,\bigU) > \UB{\pobj}.
 	\end{align}
\end{lemma} 
A proof of this result is available in \ifthenelse{\boolean{techreport}}{
	App.~\ref{sec:proofs}.
}{
	\cite[App.~\ref{tecreport-sec:proofs}]{peeling2023technicalreport}.
}
A consequence of preserving the pruning decision is that taking the new constraint ``$\pv\in [\bigLnew,\bigUnew]$'' into account at node $\nodeSymb$ does not alter the output of the \gls{bnb} procedure. 
In particular, it still correctly identifies the solutions of \eqref{prob:prob-base}.
The second requirement \eqref{eq:peeling:constraint_mnew} implies that $\node{\robj}(\bigLnew,\bigUnew)$ can possibly be larger than $\node{\robj}(\bigL,\bigU)$ since the lower bound in \eqref{eq:convex_relaxation_l0norm} is tightened by considering lower absolute values for $\bigL$ and $\bigU$.
Overall, any choice of $[\bigLnew,\bigUnew]$ verifying \eqref{eq:peeling:constraint}-\eqref{eq:peeling:constraint_mnew} thus keeps unchanged the output of the \gls{bnb} procedure while allowing for potentially more aggressive pruning decisions. 

In the rest of this section, we describe a strategy to find some interval $[\bigLnew,\bigUnew]$ satisfying~\eqref{eq:peeling:constraint}-\eqref{eq:peeling:constraint_mnew}. 
Because of the symmetry of the problem at stake, we only focus on the construction of the upper bound $\bigUnew$. 
The identification of a lower bound $\bigLnew$ can be done along the same lines.

\subsection{Target peeling strategy}

Given some index $\idxpeeled \in \setnone$ and $\peelthres>0$, we consider the following perturbed versions of \eqref{eq:node_problem_constr}:
\begin{align}
	\label{eq:peeling_condition_b}
	\node{\pobj}_\peelthres(\bigL,\bigU) \triangleq \inf_{\pv\in\kR^\pdim} 
	\node{\pfunc}(\pv;\bigL,\bigU) 
	+ \indicator(\pvi_{\idxpeeled} > \peelthres).	
\end{align}	
Problem \eqref{eq:peeling_condition_b} corresponds to~\eqref{eq:node_problem_constr} where $\pvi_{\idxpeeled}$ is additionally constrained to be strictly greater than $\peelthres$. 
The following lemma then trivially follows from the definition of $\node{\pobj}_\peelthres(\bigL,\bigU)$: 
\begin{lemma}
	\label{lemma:ideal_peeling_test}
	If $\peelthres \in [0,\bigUi_{\idxpeeled}[$ and 
	\begin{align}
		\label{eq:ideal_peeling_test}
		\node{\pobj}_\peelthres(\bigL,\bigU) > \UB{\pobj},
	\end{align}
	then \eqref{eq:peeling:constraint}-\eqref{eq:peeling:constraint_mnew} hold with 
	\begin{align} 
		\label{eq:considered_bigUnew}
		\bigUi_\idxentry' 
		\,=\,
		\begin{cases}
			\peelthres & \mbox{if $\idxentry=\idxpeeled$}\\
			\bigUi_{\idxentry} & \mbox{otherwise}.
		\end{cases}
	\end{align}
\end{lemma}
This result thus states that any $\peelthres \in [0,\bigUi_{\idxpeeled}[$ verifying \eqref{eq:ideal_peeling_test} enables to  construct some $\bigUnew$ automatically fulfilling~\eqref{eq:peeling:constraint}-\eqref{eq:peeling:constraint_mnew}. 
Unfortunately, evaluating~\eqref{eq:ideal_peeling_test} involves the same computational burden as solving~\eqref{eq:node_problem_constr}.
This problem can nevertheless be circumvented by finding some proper lower bound on $\node{\pobj}_\peelthres(\bigL,\bigU)$ as described in the next section.

\subsection{Tractable implementation}

\ifthenelse{\boolean{techreport}}{
	In App.~\ref{sec:proofs}, we leverage Fenchel-Rockafellar duality for problem~\eqref{eq:node_relaxed_problem} to show
}{
	Leveraging Fenchel-Rockafellar duality for problem~\eqref{eq:node_relaxed_problem}, it can be shown~\cite[App.~\ref{tecreport-sec:proofs}]{peeling2023technicalreport}
}%
that for any \(\dv\in\kR^\ddim\), the following lower bound on $\node{\pobj}_\peelthres(\bigL,\bigU)$ holds: 
\begin{align}
	\label{eq:lower_bound_on_peeling_metric}
	\node{\pobj}_\peelthres(\bigL,\bigU)
	\geq \node{\dfunc}(\dv;\bigL,\bigU) + \peeloffset_\idxpeeled(\dv;\bigL,\bigU)+ \peelthres \pospart{-\ktranspose{\atom}_{\idxpeeled}\dv},	
\end{align}
where
\begin{align*}
	\node{\dfunc}(\dv;\bigL,\bigU) 
	&\triangleq
	\tfrac{1}{2}\norm{\obs}{2}^2 - \tfrac{1}{2}\norm{\obs - \dv}{2}^2+ \reg|\setone|  \\
	&\qquad- \sum_{\idxentry\in\setone} 
	\pivot{0,\idxentry}(\ktranspose{\atom}_{\idxentry}\dv)
	- \sum_{\idxentry\in\setnone} \pivot{\reg,\idxentry}(\ktranspose{\atom}_{\idxentry}\dv)
	\\
	\peeloffset_\idxpeeled(\dv;\bigL,\bigU)
	&\triangleq 
	\pivot{\reg,\idxpeeled}(\ktranspose{\atom}_{\idxpeeled}\dv)
	- \bigUi_{\idxpeeled}\pospart{\ktranspose{\atom}_{\idxpeeled}\dv}+ \reg 
\end{align*}
and $\pivot{\rho,\idxentry}(\cvi)\triangleq \pospart{\bigUi_{\idxentry} \cvi  - \rho } + \pospart{\bigLi_{\idxentry}\cvi  -\rho}$.	

Using this result, condition \eqref{eq:ideal_peeling_test} can be relaxed as
\begin{align}\label{eq:alpha_peeling_condition}
	\node{\dfunc}(\dv;\bigL,\bigU) + \peeloffset_\idxpeeled(\dv;\bigL,\bigU)+ \peelthres \pospart{-\ktranspose{\atom}_{\idxpeeled}\dv }	>\UB{\pobj}
	.
\end{align}
Hence, choosing any \(\peelthres\in[0,\bigUi_\idxpeeled[\) verifying~\eqref{eq:alpha_peeling_condition} for some \(\dv\in\kR^\ddim\) defines a new valid constraint via \eqref{eq:considered_bigUnew}, in the sense of \eqref{eq:peeling:constraint}-\eqref{eq:peeling:constraint_mnew}. 
Interestingly, the left-hand side of~\eqref{eq:alpha_peeling_condition} depends linearly on \(\peelthres\), thus allowing to precisely characterize the range of possible values satisfying the strict inequality~\eqref{eq:alpha_peeling_condition}.
This leads us to the main result of this section.
\begin{proposition}
	\label{prop:entrywise_peeling}
	Let \(\dv\in\kR^\ddim\).
	If \(\transp{\atom}_{\idxpeeled}\dv\geq 0\) and 
	\begin{align}
		\label{eq:peeling:positive_case}
		\node{\dfunc}(\dv;\bigL,\bigU) + \peeloffset_\idxpeeled(\dv;\bigL,\bigU)>\UB{\pobj},
	\end{align}
	then $\bigUnew$ defined as in~\eqref{eq:considered_bigUnew} with $\peelthres=0$ fulfills~\eqref{eq:peeling:constraint}-\eqref{eq:peeling:constraint_mnew}.
	Moreover, if \(\transp{\atom}_{\idxpeeled}\dv<0\), then 
	$\bigUnew$ defined as in~\eqref{eq:considered_bigUnew} with any 
	\(\peelthres\in[0,\bigUi_{\idxpeeled}[\) verifying
	\begin{equation} 
		\label{eq:entrywise_peeling_treshold}
		\peelthres
		>
		\UB{\peelthres}\triangleq 
		\frac{
			\UB{\pobj} - \node{\dfunc}(\dv;\bigL,\bigU) - \peeloffset_\idxpeeled(\dv;\bigL,\bigU) 
		}{
			\pospart{-\ktranspose{\atom}_\idxpeeled\dv} 
		}
	\end{equation}
	fulfills~\eqref{eq:peeling:constraint}-\eqref{eq:peeling:constraint_mnew}.
\end{proposition}

Our next result shows that \Cref{prop:entrywise_peeling} can be applied to all indices \(\idxpeeled\in\intervint{1}{\pdim}\) either sequentially or in parallel, while preserving the correctness of the \gls{bnb} procedure:
\begin{lemma} 
	\label{lemma:parallel_peeling}
	Let \([\bigL',\bigU']\) and \([\bigL'',\bigU'']\) be two intervals satisfying~\eqref{eq:peeling:constraint}-\eqref{eq:peeling:constraint_mnew}.
	Then, the interval \([\bigL',\bigU']\cap[\bigL'',\bigU'']\) also fulfills~\eqref{eq:peeling:constraint}-\eqref{eq:peeling:constraint_mnew}.
\end{lemma}
A proof is available in \ifthenelse{\boolean{techreport}}{
	App.~\ref{sec:proofs}.
}{
	\cite[App.~\ref{tecreport-sec:proofs}]{peeling2023technicalreport}.
}
We note that in terms of complexity the parallel application of \Cref{prop:entrywise_peeling} to all indices \(\idxpeeled\in\intervint{1}{\pdim}\) requires the computation of the inner products $\{\transp{\atom}_{\idxentry}\dv\}_{\idxentry=1}^{\pdim}$ and \textit{one single} evaluation of \(\node{\dfunc}(\dv;\bigL,\bigU)\).
Interestingly, these inner products are already computed in most numerical procedures solving~\eqref{eq:node_relaxed_problem} and are thus usually available at no additional cost, see \eg~\cite[Sec.~4.3]{guyard2022node}.
The overhead complexity of applying in parallel our proposed peeling strategy thus scales as \(\mathcal{O}(\pdim+\ddim)\). 

\subsection{Propagating peeling down the tree}
\label{sec:propagation}

In this section, we emphasize that any interval $[\bigLnew,\bigUnew]$ verifying \eqref{eq:peeling:constraint}-\eqref{eq:peeling:constraint_mnew} at node $\nodeSymb$ can be used as a starting point to apply our peeling procedure at the child nodes of $\nodeSymb$. 
More specifically, the following result holds: 

\begin{lemma}
	\label{corollary:heritage}
	Let $[\bigLnew,\bigUnew]$ be some interval verifying \eqref{eq:peeling:constraint}-\eqref{eq:peeling:constraint_mnew} at node $\nodeSymb$ and let $\nodeSymb'$ be some child node of $\nodeSymb$. 
	Assume that the peeling procedure defined in \Cref{prop:entrywise_peeling} is applied at node $\nodeSymb'$ with $[\bigLnew,\bigUnew]$  as input, rather than $[\bigL,\bigU]$, to generate a new interval $[\bigLnew',\bigUnew']$.
	Then we have 
	\begin{subequations}
	\begin{align}
		\forall \pv\in[\bigL,\bigU]\setminus [\bigLnew',\bigUnew']&:\,
		\pfunc^{\nodeSymb'}(\pv;\bigL,\bigU)>\UB{\pobj}
		\label{eq:peeling:constraint_3}
		\\
		[\bigLnew',\bigUnew']&\subseteq [\bigL,\bigU]
		\label{eq:peeling:constraint_mnew_3}
		.
	\end{align}
	\end{subequations}
\end{lemma}
A proof of this result is available in \ifthenelse{\boolean{techreport}}{
	App.~\ref{sec:proofs}.
}{
	\cite[App.~\ref{tecreport-sec:proofs}]{peeling2023technicalreport}.
} 
In other words, \Cref{corollary:heritage} states that any peeled interval $[\bigLnew,\bigUnew]$ computed at node $\nodeSymb$ can be used as a starting point to apply a new peeling step at any child node $\nodeSymb'$. 
This allows to propagate the  peeled interval $[\bigLnew,\bigUnew]$ down the decision tree to hopefully improve sequentially the tightness of the convex relaxation \eqref{eq:node_relaxed_problem}.
 

\section{Numerical results} 
\label{sec:results}

This section reports an empirical study demonstrating the effectiveness of the proposed peeling procedure to accelerate the resolution of~\eqref{prob:prob-base} on a synthetic dataset.
\ifthenelse{\boolean{techreport}}{
	Additional simulation results can be found in App.~\ref{sec:additional numerical results}.
}{}
 We refer the reader to~\cite{Bertsimas2016,hastie2017extended} for an in-depth study of the statistical properties of the optimizer obtained from this problem.

\subsection{Experimental setup}

We consider instances of problem~\eqref{prob:prob-base} with dimensions \((\ddim,\pdim)=(100,150)\).
For each trial, new realizations of \(\dic,\obs\) and \(\lambda\) are generated as follows.
Each row of the dictionary \(\dic\) is drawn from a multivariate normal distribution with zero mean and covariance matrix \({\corrmatrix}\in\kR^{\pdim\times \pdim}\). 
The $(i,j)$th entry of $\corrmatrix$ is defined as 
\ifthenelse{\boolean{techreport}}{\(\corrmatrixel_{ij}=\corrparam^{\abs{i-j}}\), \(\forall i,j\in\intervint{1}{\pdim}\), with $\corrparam=0.1$.}{\(\corrmatrixel_{ij}=10^{-\abs{i-j}}\), \(\forall i,j\in\intervint{1}{\pdim}\).}
Each realization of \(\obs\) is generated in two steps.
We first create \ifthenelse{\boolean{techreport}}{a $\sparsitylevel$-sparse vector \(\groundtruth \in \kR^{\pdim}\) with evenly-distributed non-zero components, where $\sparsitylevel=5$.}{a \(5\)-sparse vector \(\groundtruth \in \kR^{\pdim}\) with evenly-distributed non-zero components.} 
The non-zero entries are defined as \(\groundtruthi_{\idxentry} = \sign{r_{\idxentry}} + r_{\idxentry}\) where  \(r_{\idxentry}\) is an independent realization of a zero-mean Gaussian with variance $\stdnnz^2$. 
We then set \(\obs = \dic\groundtruth + \noise\) for some  zero-mean white Gaussian noise \(\noise\). The variance of the noise is adjusted so that the $\mathrm{SNR} \triangleq 10 \log_{10}(\norm{\dic\groundtruth}{2}^2/\norm{\noise}{2}^2)$ is equal to \(15\)dB.
The parameter \(\reg\) is calibrated for each instance of~\eqref{prob:prob-base} using the cross-validation tools of the \textsc{L0Learn} package \cite{hazimeh2020fast} with the default parameters\ifthenelse{\boolean{techreport}}{. 
	More specifically, we call the \texttt{cv.fit} procedure that takes \(\obs\) and \(\dic\) as inputs and returns couples $(\pv_{\reg_i},c_{\reg_i})$ from a grid of values $\{\reg_i\}_{i \in \kN}$ selected data-dependently by the package.
	The vector $\pv_{\reg_i}$ is an approximate solution of~\eqref{prob:prob-base} where the weight of the $\ell_0$-norm is set to $\reg_i$ and \(c_{\reg_i}\) is an associated cross-validation score 
	computed on 10 folds of $\ddim/10$ randomly sampled rows in $\dic$ and entries in $\obs$.
	We then set
	\begin{equation}
		\reg = \kargmin_{\reg_i} \ c_{\reg_i} 
		\ \text{s.t.} \ 
		\kvvbar{\pv_{\reg_i}}_0 = \kvvbar{\groundtruth}_0
		.
	\end{equation}
}{, see \cite[Sec.~\ref{tecreport-sec:results}]{peeling2023technicalreport} for more details.} 

\subsection{Competing procedures}
We consider the following numerical solvers addressing~\eqref{prob:prob-base}: 
\textit{i)} \textsc{Cplex}~\cite{cplex2009v12}, a generic mixed-integer problem solver;
\textit{ii)} \textsc{L0bnb}~\cite{hazimeh2021sparse}, a standard \gls{bnb} procedure using a ``breadth-first search'' exploration strategy, see \cite[Sec.~3.3]{hazimeh2021sparse};  
\textit{iii)} \textsc{Sbnb}, a standard \gls{bnb} procedure using a ``depth-first search'' exploration strategy, see ~\cite[Sec.~2.2]{mhenni2020sparse};  
\textit{iv)} \textsc{Sbnb-N}, corresponding to \textsc{Sbnb} enhanced with additional ``node-screening'' techniques, see~\cite{guyard2022node};
\textit{v)} \textsc{Sbnb-P}, corresponding to \textsc{Sbnb} enhanced with the peeling strategy presented in this paper. 
\textsc{L0bnb}, \textsc{Sbnb}, \textsc{Sbnb-N} and \textsc{Sbnb-P} all use the same solving procedure for relaxed problem \eqref{eq:node_relaxed_problem}, namely a coordinate descent method~\cite{wright2015coordinate}.  
We use the C++ implementation of \textsc{Cplex}\footnote{
	\protect{\url{https://github.com/jump-dev/CPLEX.jl}}
}
and the Python implementation of \textsc{L0bnb}.\footnote{
	\protect{\url{https://github.com/hazimehh/L0Learn}}
}
\textsc{Sbnb}, \textsc{Sbnb-N} and \textsc{Sbnb-P} are implemented in Julia.\footnote{
	\protect{\url{https://github.com/TheoGuyard/BnbPeeling.jl}}
}

For \textsc{Sbnb-P}, peeling is applied at each iteration of the numerical 
procedure solving the relaxed problem \eqref{eq:node_relaxed_problem}.
We use the current iterate, say $\pv^{(k)}$, to define $\dv \triangleq \obs-\dic \pv^{(k)}$ 
and apply the peeling rules defined in \Cref{prop:entrywise_peeling} in parallel, \ie simultaneously for all the components of $\pv$.
The value of $\peelthres$ satisfying \eqref{eq:entrywise_peeling_treshold}, if any, is chosen as $\peelthres=\UB{\peelthres}+10^{-16}$. 
The peeled intervals are propagated through the decision tree as described in \Cref{sec:propagation}.

All the solving procedures are provided with the initial bounds $\bigL=-\bigM\1$ and $\bigU=\bigM\1$ for some proper value of $\bigM$.
This corresponds to the standard \textit{``Big-$M$''} constraint commonly considered in the literature \cite{mhenni2020sparse,ben2022global,bertsimas2021unified,bourguignon2015exact}. 
As far as our random simulation setup is concerned, it can be shown that \eqref{prob:prob-base} admits a unique minimizer $\opt{\pv}$ with probability one and  we thus choose $\bigM= \bigMfactor \norm{\opt{\pv}}{\infty}$ for some $\bigMfactor\geq 1$ in our simulations.
This requires to solve \eqref{prob:prob-base} once beforehand to identify $\opt{\pv}$.
This operation is here only done for the sake of comparing the sensibility of the solving methods to the choice of $\bigMfactor$.
\ifthenelse{\boolean{techreport}}{
	In practice, we obtain $\opt{\pv}$ by solving a sequence of problems with an increasing value of $\bigM$ in the \textit{Big-M} constraint.
	More specifically, letting $\opt{\pv}_{\bigM}$ denote the solution of \eqref{prob:prob-base}  with the additional constraint ``$-\bigM\1 \leq \pv \leq \bigM\1$'', we are guaranteed that $\opt{\pv}=\opt{\pv}_{\bigM}$ as soon as the strict inequality $\|\opt{\pv}_{\bigM}\|_{\infty} < \bigM$ holds. 
	We thus compute the solution $\opt{\pv}_{\bigM}$ for a 
	sequence of $\bigM$ of the form 
	$\{\eta^i\bigM_0\}_{i\in\kN}$ with $\eta = 1.1$ and for some $\bigM_0>0$ and stop as soon as $\|\opt{\pv}_{\bigM}\|_{\infty} < \bigM$.
}{
	More details are given in our companion paper~\cite[Sec.~\ref{tecreport-sec:results}]{peeling2023technicalreport}.
}

\subsection{Computational gains}

\begin{figure}[t]

\begin{tikzpicture}
	\begin{groupplot}[
		group style = {
			group size 			= 2 by 2,
			horizontal sep 		= 1cm,
			vertical sep 		= 1cm,
		},
		height				= 4cm,
		width				= 0.28\textwidth,
		grid 				= both,
		cycle list name 	= mylist,
		legend cell align 	= left,
		legend pos 			= north west,
		legend style = {
			fill			= white, 
			fill opacity	= 0.6, 
			draw opacity 	= 1, 
			text opacity 	= 1,
			font 			= \scriptsize,
		},
		xtick distance		= 1,
	]

		\nextgroupplot[
			title 			= {\small Solving time (sec.)},
			ymode 			= log,
			extra y ticks	= {10, 1000},
			xlabel 			= $\bigMfactor$,
		]
		\foreach \solver in {Cplex,L0bnb,Sbnb,SbnbNscr,SbnbPeel}{
			\addplot table [
				x		= bigmfactor,
				y		= \solver,
				col sep	= comma,
			] {dat/synt_opt_linear_LeastSquares_Bigm_5_100_150_1.0_0.1_1000.0_x_solve_time.csv};
			\label{solver:xp1:\solver}
		}
		\coordinate (top) at (rel axis cs:0,1);

		\nextgroupplot[
			title 	= {\small Gain},
			xlabel 	= $\bigMfactor$,
			ymin	= 0.9,
			ymax	= 2.1,
		]
		\addplot[smooth,ultra thick,black,mark=*,mark size=0.75pt] table [
			x		= bigmfactor,
			y expr	= {\thisrow{SbnbNscr}/\thisrow{SbnbPeel})},
			col sep	= comma,
		] {dat/synt_opt_linear_LeastSquares_Bigm_5_100_150_1.0_0.1_1000.0_x_solve_time.csv};
		\addplot[smooth,ultra thick,densely dashed] table [
			x		= bigmfactor,
			y expr	= {\thisrow{SbnbNscr}/\thisrow{SbnbPeel})},
			col sep	= comma,
		] {dat/synt_opt_linear_LeastSquares_Bigm_5_100_150_1.0_0.1_1000.0_x_node_count.csv};

		\nextgroupplot[
			ymode 			= log,
			extra y ticks	= {10, 1000},
			xlabel 			= $\stdnnz$,
		]
		\foreach \solver in {Cplex,L0bnb,Sbnb,SbnbNscr,SbnbPeel}{
			\addplot table [
				x		= sigma,
				y		= \solver,
				col sep	= comma,
			] {dat/synt_opt_linear_LeastSquares_Bigm_5_100_150_x_0.1_1000.0_5.0_solve_time.csv};
		}

		\nextgroupplot[
			xlabel 	= $\stdnnz$,
			ymin	= 0.9,
			ymax	= 2.1,
		]
		\addplot[smooth,ultra thick,black,mark=*,mark size=0.75pt] table [
			x		= sigma,
			y expr	= {\thisrow{SbnbNscr}/\thisrow{SbnbPeel})},
			col sep	= comma,
		] {dat/synt_opt_linear_LeastSquares_Bigm_5_100_150_x_0.1_1000.0_5.0_solve_time.csv};
		\addplot[smooth,ultra thick,densely dashed] table [
			x		= sigma,
			y expr	= {\thisrow{SbnbNscr}/\thisrow{SbnbPeel})},
			col sep	= comma,
		] {dat/synt_opt_linear_LeastSquares_Bigm_5_100_150_x_0.1_1000.0_5.0_node_count.csv};
		\coordinate (bot) at (rel axis cs:1,0);
	\end{groupplot}

	\path (top|-current bounding box.north)-- coordinate(legendpos) (bot|-current bounding box.north);
	\matrix[
		matrix of nodes,
		anchor=south,
		draw,
		inner sep=0.2em,
		every node/.style={anchor=base west, font=\small},
	] at([yshift=1ex]legendpos)
	{
		\ref{solver:xp1:Cplex} & \textsc{Cplex} & 
		\ref{solver:xp1:L0bnb} & \textsc{L0bnb} & 
		\ref{solver:xp1:Sbnb} & \textsc{Sbnb} & 
		\ref{solver:xp1:SbnbNscr} & \textsc{Sbnb-N} & 
		\ref{solver:xp1:SbnbPeel} & \textbf{\textsc{Sbnb-P}} \\
	};
\end{tikzpicture}
	\vspace*{-0.75cm}
    \caption{
		Left: Solving time as a function of $\bigMfactor$  (top, $\stdnnz=1$) and $\stdnnz$ (bottom, $\bigMfactor=5$). 
		Right: gain in terms of solving time (solid) and number of nodes explored (dashed) with respect to \textsc{Sbnb-N}.
	}
	\label{fig:sensibility}
	\vspace*{-0.5cm}
\end{figure}
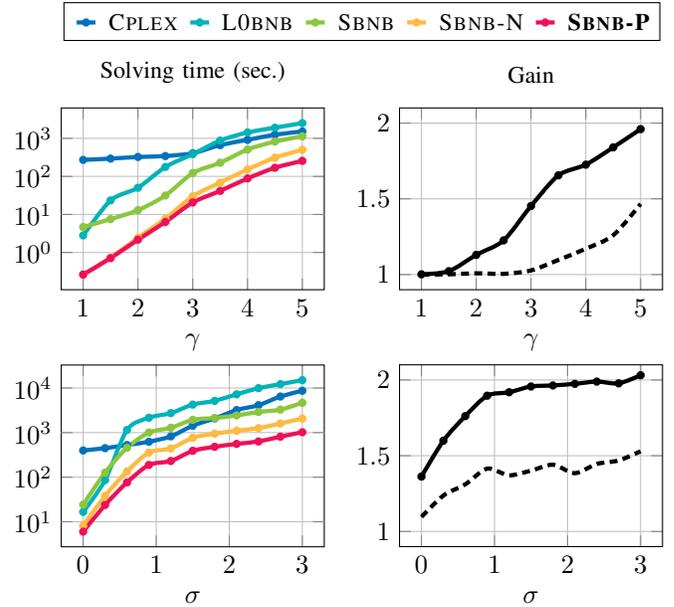

\Cref{fig:sensibility} presents the performance of the considered solving procedures. All results are averaged over 50 problem instances. Experiments were run on one Intel Xeon E5-2660 v3 CPU clocked at 2.60 GHz with 16 GB of RAM.
The left column in \Cref{fig:sensibility} represents the average solving time of each procedure as a function of $\bigMfactor$ (top) and $\stdnnz$ (bottom);
the right column illustrates the gain allowed by the proposed method in terms of solving time (solid) and number of nodes explored (dashed) as compared to its best competitor, that is \textsc{Sbnb-N}.    

We note that \textsc{Sbnb-P} leads to the smallest running time in  all the considered setups. 
Since the latter corresponds to \textsc{Sbnb} where peeling has been added, the spacing between the red and green curves materializes the gain provided by peeling. 
As far as our simulation setup is concerned, we see that the proposed method enables an acceleration of almost one order of magnitude with respect to \textsc{Sbnb}. 
It is noticeable that this acceleration occurs even if $\bigMfactor=1$, that is the \textit{Big-$M$} constraint is perfectly tuned to the problem at hand. This is due to the fact that peeling can refine \textit{individually} each component of the initial bounds $\bigL$ and $\bigU$ at \textit{each node} of the \gls{bnb} decision tree to fit the local geometry of the problem. 

We also notice that \textsc{Sbnb-P} improves over \textsc{Sbnb-N}, which can be seen as another acceleration of \textsc{Sbnb}. 
In particular, \textsc{Sbnb-P} performs always as well as \textsc{Sbnb-N} as emphasized by the gains in the right-hand side of Fig.~\ref{fig:sensibility}. 
We note in particular the gain provided by peeling in terms of number of nodes processed by the \gls{bnb} procedure: as expected, peeling allows for more aggressive pruning and thus reduces the number of nodes to be explored.

\section{Conclusion}

In this paper, we presented a tractable strategy, named ``\emph{peeling}'', to tighten the box constraints used in a \gls{bnb} procedure tailored to $\ell_0$-regularized least-squares problems. Unlike the standard approach which imposes \textit{one global} constraint to the problem, our strategy aims to locally refine the box constraints \textit{at each node} of the decision tree. 
This refinement enables to strengthen the convex relaxations used in the pruning decisions made by the \gls{bnb} procedure and can lead to significant improvements in terms of solving time, as emphasized by our simulation results.

\clearpage

\bibliographystyle{IEEEtran}
\bibliography{main.bib}

\ifthenelse{\boolean{techreport}}{
    \clearpage
    \appendices

\section{Proof of the main results}
\label{sec:proofs}

\subsection{Proof of \texorpdfstring{Lemma~\ref{lemma: consequence properties peeled interval}}{Lemma 1}}  
\label{proof:lemma:consequence_properties_peeled_interval}

We first have 
	\begin{align*}
		\node{\pobj}(\bigLnew,\bigUnew)\geq \node{\pobj}(\bigL,\bigU)
	\end{align*}
since $[\bigLnew,\bigUnew]\subseteq [\bigL,\bigU]$ from \eqref{eq:peeling:constraint_mnew} and the left-hand side corresponds to the minimum value of a problem more constrained than the right-hand side.
The reverse implication in \eqref{eq:peeling:equivalence} is thus always verified. 

The direct implication results from the following observations.	
If
$\node{\pobj}(\bigL,\bigU)\leq \UB{\pobj}$, then we have 
	\begin{align*}
		\node{\pobj}(\bigL,\bigU) = \node{\pobj}(\bigLnew,\bigUnew)
	\end{align*}
since \eqref{eq:peeling:constraint} ensures that the  minimizers of $\node{\pfunc}(\pv;\bigL,\bigU)$ do not belong to $[\bigL,\bigU]\setminus[\bigLnew,\bigUnew]$ and are therefore also minimizers of $\node{\pfunc}(\pv;\bigLnew,\bigUnew)$. 
We thus obtain the direct implication by contraposition.

\subsection{Proof of~\texorpdfstring{\eqref{eq:lower_bound_on_peeling_metric}}{(16)}} \label{proof:lower_bounds}

We first note that
\begin{subequations}
	\begin{align}
		\node{\pobj}_\peelthres(\bigL,\bigU) &= \inf_{\pv\in\kR^\pdim} \node{\pfunc}(\pv;\bigL,\bigU) + \indicator(\pvi_{\idxpeeled} > \peelthres) \\
		&\geq \inf_{\pv\in\kR^\pdim} \node{\rfunc}_{\idxpeeled}(\pv;\bigL,\bigU) + \indicator(\pvi_{\idxpeeled} > \peelthres)\\
		&\geq \inf_{\pv\in\kR^\pdim} \node{\rfunc}_{\idxpeeled}(\pv;\bigL,\bigU) + \indicator(\pvi_{\idxpeeled} \geq \peelthres) \\
		&= \min_{\pv\in\kR^\pdim} \node{\rfunc}_{\idxpeeled}(\pv;\bigL,\bigU) + \indicator(\pvi_{\idxpeeled} \geq \peelthres)\label{eq:relaxed_perturbed_problem_b}
	\end{align}
\end{subequations}
where 
\begin{multline} 
	\node{\rfunc}_{\idxpeeled}(\pv;\bigL,\bigU)
	\triangleq 
	\tfrac{1}{2}\kvvbar{\obs - \dic\pv}_2^2 
	+ \reg \sum_{\idxentry \in \setnone\backslash \idxpeeled} \Big(\tfrac{\pospart{\pvi_{\idxentry}}}{\bigUi_\idxentry}
	- \tfrac{\pospart{-\pvi_{\idxentry}}}{\bigLi_\idxentry}\Big)
	\\
	+ \reg (\card{\setone}+1)
	+ \indicator(\subzero{\pv} = \0)
	+ \indicator(\pv \in [\bigL,\bigU]) 
\end{multline}
and the last equality follows from the fact that the objective function is lower-semi-continuous and its domain is closed and bounded. 

The lower bound in \eqref{eq:lower_bound_on_peeling_metric} then stems from the Fenchel-Rockafellar dual problem of \eqref{eq:relaxed_perturbed_problem_b}.
More specifically,  we first notice that \eqref{eq:relaxed_perturbed_problem_b} can be rewritten as
\begin{equation}\label{eq:primal_formulation}
	\min_{\pv\in\kR^\pdim} f(\dic\pv) 
	+ g(\pv)
\end{equation}
with
\begin{subequations}
	\begin{align}
		\label{eq:def-f-peel}
		f(\mathbf{z}) \,=\,& \tfrac{1}{2}\kvvbar{\obs - \mathbf{z}}_2^2
		\\
		\label{eq:def-g-peel}
		g(\pv) \,=\,& g_\idxpeeled(\pvi_\idxpeeled) + \sum_{\idxentry\neq \idxpeeled} 
		g_\idxentry(\pvi_\idxentry)
	\end{align}
\end{subequations}
and
\begin{align*}
	g_\idxpeeled(\pvi)
	&= \indicator(\pvi \in [\peelthres,\bigUi_{\idxpeeled}]) + \lambda\\
	g_\idxentry(\pvi)
	&=
	\begin{cases}
		\indicator(\pvi=0) 	&\text{ if } \idxentry\in\setzero \\
		\indicator(\pvi \in [\bigLi_{\idxentry},\bigUi_{\idxentry}]) + \lambda &\text{ if } \idxentry\in\setone \\
		\indicator(\pvi \in [\bigLi_{\idxentry},\bigUi_{\idxentry}]) + \lambda(\tfrac{\pospart{\pvi}}{\bigUi_\idxentry} - \tfrac{\pospart{-\pvi}}{\bigLi_\idxentry}) & \text{ if } \idxentry\in\setnone\backslash \idxpeeled.
	\end{cases}
\end{align*}
Second, we have from standard results of convex optimization~\cite[Ch.~12]{Beck2017aa} that the Fenchel-Rockafellar dual problem of~\eqref{eq:primal_formulation} reads
\begin{equation}
	\label{eq:generic_dual}
	\max_{\dv\in\kR^\ddim} \ -\conj{f}(-\dv) - \conj{g}(\transp{\dic}\dv),
\end{equation}
where $\conj{f}$ and $\conj{g}$ denote the convex conjugates of \(f\) and \(g\) (see~\cite[Def.~4.1]{Beck2017aa}).
In particular, we have \(\forall \pv\in\kR^\pdim, \dv\in\kR^\ddim\): 
\begin{align}\label{eq:duality}
	f(\dic\pv) + g(\pv)
	\geq
	-\conj{f}(-\dv) - \conj{g}(\transp{\dic}\dv)
\end{align}
as a consequence of weak duality applied to the pair~\eqref{eq:primal_formulation}-\eqref{eq:generic_dual}. 
Our lower bound in \eqref{eq:lower_bound_on_peeling_metric}  corresponds to a particularization of the right-hand side of \eqref{eq:duality} to \eqref{eq:def-f-peel}-\eqref{eq:def-g-peel}. 
In particular, applying the definition of the convex conjugate to the function $f$, one easily obtains:
\begin{equation}\label{eq:conj f}
	\conj{f}(\dv) = -\tfrac{1}{2}\norm{\obs}{2}^2 + \tfrac{1}{2}\norm{\obs + \dv}{2}^2
	.
\end{equation}
Invoking \Cref{lemma:conj-setone} in \Cref{app:proof} with \((b,\bigLi,\bigUi)=(\lambda,\peelthres,\bigUi_\idxpeeled)\), we obtain
\begin{equation}
	\label{eq:conj_g2}
	\conj{g}_\idxpeeled(\cvi)
	=
	\bigUi_{\idxpeeled}\pospart{\cvi } - \peelthres \pospart{-\cvi } - \reg.
\end{equation}
Applying again \Cref{lemma:conj-setone} with
\begin{equation*}
	(b,\bigLi,\bigUi) = 
	\begin{cases}
		 (0,0,0) &\text{if} \ \idxentry \in \setzero \\
		(\reg,\bigLi_{\idxentry},\bigUi_{\idxentry}) &\text{if} \ \idxentry \in \setone,
	\end{cases}
\end{equation*} 
leads to 
\begin{equation}
	\label{eq:conj_g1}
	\conj{g}_\idxentry(\cvi)
	=
	\begin{cases}
		0 	&\text{ if } \idxentry\in\setzero \\
		\pospart{\bigUi_{\idxentry}\cvi } + \pospart{\bigLi_{\idxentry}\cvi} - \reg &\text{ if } \idxentry\in\setone 
	\end{cases}
\end{equation}
where we have used the fact that for all \(\idxentry \in \setone\), \(\bigUi_{\idxentry}\pospart{\cvi}=\pospart{\bigUi_{\idxentry}\cvi}\) (resp. \(\bigLi_{\idxentry}\pospart{-\cvi}=-\pospart{\bigLi_{\idxentry}\cvi}\)) since \(\bigUi_{\idxentry}\geq0\) (resp. \(\bigLi_{\idxentry} \leq0\)). 	
Similarly, using \Cref{lemma:conj-setnone} in \Cref{app:proof} with $(a,\bigLi,\bigUi) = (\reg,\bigLi_{\idxentry},\bigUi_{\idxentry}) $ 
we obtain: 
\begin{equation}
	\label{eq:conj_g3}
	\conj{g}_\idxentry(\cvi)
	=
	\pospart{\bigUi_{\idxentry}\cvi  - \reg } + \pospart{\bigLi_{\idxentry}\cvi  -\reg} \quad \text{ if } \idxentry\in\setnone\backslash \idxpeeled.
\end{equation}	
Finally, combining \eqref{eq:conj f}-\eqref{eq:conj_g3} leads to the desired result. \\

\subsection{Proof of \texorpdfstring{Lemma~\ref{lemma:parallel_peeling}}{Lemma~3}} 
\label{sec:proof_lemma_parallel_application}

Let \([\bigL',\bigU']\) and \([\bigL'',\bigU'']\) be two intervals satisfying~\eqref{eq:peeling:constraint}-\eqref{eq:peeling:constraint_mnew}, and defines \(\mathcal{S}=[\bigL',\bigU']\cap[\bigL'',\bigU'']\).
We first note that \([\bigL',\bigU']\cap[\bigL'',\bigU'']\) can be rewritten as
\begin{equation*}
	\mathcal{S}
	=
	\bigtimes_{\idxentry=1}^\pdim \kintervcc{
		\max(\bigLi'_\idxentry, \bigLi''_\idxentry)
	}{
		\min(\bigUi'_\idxentry, \bigUi''_\idxentry)
	}
\end{equation*}
where \(\bigtimes\) denotes the Cartesian product of sets, and therefore defines an interval.
It thus remains to show that \(\mathcal{S}\) fulfills~\eqref{eq:peeling:constraint}-\eqref{eq:peeling:constraint_mnew}.

First, the inclusion \(\mathcal{S}\subset[\bigL,\bigU]\) follows from the fact that both \([\bigL',\bigU']\) and \([\bigL'',\bigU'']\) verify \eqref{eq:peeling:constraint_mnew}. 
Second, the fact that \(\mathcal{S}\) fulfills~\eqref{eq:peeling:constraint} is a consequence of the following set equality: 
\begin{equation*}
	[\bigL,\bigU]\setminus\mathcal{S} = \kparen{[\bigL,\bigU]\setminus[\bigL',\bigU']}\cup \kparen{[\bigL,\bigU]\setminus[\bigL'',\bigU'']}
	.
\end{equation*}
Hence, if \(\pv\in[\bigL,\bigU]\setminus\mathcal{S}\) then \(\pv\in\kparen{[\bigL,\bigU]\setminus[\bigL',\bigU']}\) or \(\pv\in\kparen{[\bigL,\bigU]\setminus[\bigL'',\bigU'']}\).
Assume without loss of generality that \(\pv\in\kparen{[\bigL,\bigU]\setminus[\bigL',\bigU']}\).
Since \([\bigL',\bigU']\) fulfills~\eqref{eq:peeling:constraint}, we then have \(\node{\pfunc}(\pv;\bigL,\bigU)>\UB{\pobj}\).
One concludes the proof by noting that the latter rationale holds for all \(\pv\in[\bigL,\bigU]\setminus\mathcal{S}\).

\subsection{Proof of {Lemma}~\texorpdfstring{\ref{corollary:heritage}}{4}} \label{sec:proof:corollary:heritage}

If the peeling procedure defined in \Cref{prop:entrywise_peeling} is applied at node $\nodeSymb'$ with $[\bigLnew,\bigUnew]$  as input to generate a new interval $[\bigLnew',\bigUnew']$, we have by construction:
	\begin{subequations}
	\begin{align}
		\forall \pv\in[\bigLnew,\bigUnew]\setminus [\bigLnew',\bigUnew']&:\,
		\pfunc^{\nodeSymb'}(\pv;\bigLnew,\bigUnew)>\UB{\pobj}
		\label{eq:peeling:constraint_2}
		\\
		[\bigLnew',\bigUnew']&\subseteq [\bigLnew,\bigUnew]
		.
		\label{eq:peeling:constraint_mnew_2}
	\end{align}
	\end{subequations}
We note that \eqref{eq:peeling:constraint_2} is equivalent to 
 	\begin{align}\label{eq:peeling:constraint_2b}
		\forall \pv\in[\bigLnew,\bigUnew]\setminus [\bigLnew',\bigUnew']:\,
		\pfunc^{\nodeSymb'}(\pv;\bigL,\bigU)>\UB{\pobj}
	\end{align}
since $\pv\in [\bigLnew,\bigUnew]$ and $[\bigLnew,\bigUnew]\subseteq[\bigL,\bigU]$ by hypothesis. 

Moreover, we have 
	\begin{align*}
		\pfunc^{\nodeSymb'}(\pv;\bigL,\bigU) \geq \node{\pfunc}(\pv;\bigL,\bigU)
		.
	\end{align*}
since $\mathcal{X}^{\nodeSymb'} \subset \node{\mathcal{X}}$ for any child node $\nodeSymb'$ of $\nodeSymb$. 
Hence, any interval $[\bigLnew,\bigUnew]$ verifying \eqref{eq:peeling:constraint} at node $\nodeSymb$ obviously also fulfills \eqref{eq:peeling:constraint} at $\nodeSymb'$, that is 
	\begin{align}\label{eq:peeling:constraint ho ho}
		\forall \pv\in[\bigL,\bigU]\setminus [\bigLnew,\bigUnew]:\,
		\pfunc^{\nodeSymb'}(\pv;\bigL,\bigU)>\UB{\pobj}
		.
	\end{align}
Combining \eqref{eq:peeling:constraint_mnew_2}, \eqref{eq:peeling:constraint_2b} and \eqref{eq:peeling:constraint ho ho}, we finally obtain the result. 

\section{Technical lemmas} 
\label{app:proof}

In this appendix, we derive the expressions of two conjugate functions appearing in the derivation of \eqref{eq:lower_bound_on_peeling_metric}. The results are encapsulated in \Cref{lemma:conj-setone,lemma:conj-setnone} below. 

\begin{lemma}
	\label{lemma:conj-setone}
	Let \(\kfuncdef{\regfunc}{\kR}{\kR}\) be defined for all \(\pvi\in\kR\) as
	\begin{equation*}
		\regfunc(\pvi) = 
		b + \indicator(\pvi\in[\bigLi,\bigUi])		
	\end{equation*}
	for $b\in\kR$ and \(\bigLi \leq \bigUi\). 
	Then, 
	\begin{equation*}
		\regfunc^\star(\cvi) =
		\bigUi \pospart{\cvi} - \bigLi \pospart{-\cvi} -b
		.
	\end{equation*}	
\end{lemma}

\begin{proof}
	By definition of the convex conjugate (see~\cite[Def.~4.1]{Beck2017aa}), one has \(\forall \cvi\in\kR\):
	\begin{subequations}
	\begin{align*}
		\conj{\regfunc}(\cvi) 
		&= \sup_{\pvi\in\kR} \; \cvi\pvi - g(\pvi)\\
		&= \max_{\bigLi\leq \pvi \leq \bigUi} \; \cvi\pvi - b\\
		&= \bigUi \pospart{\cvi} - \bigLi \pospart{-\cvi} -b
	\end{align*}
	\end{subequations}
\end{proof}

\newcommand{\funcsymbAppBLemmaTwo}{\regfunc}
\begin{lemma}
	\label{lemma:conj-setnone}
	Let \(\kfuncdef{\funcsymbAppBLemmaTwo}{\kR}{\kR}\) be defined for all \(\pvi\in\kR\) as
	\begin{equation*}
		\funcsymbAppBLemmaTwo(\pvi) = 
		\indicator(\pvi\in[\bigLi,\bigUi])		
		+\tfrac{a}{\bigUi} \pospart{\pvi} 
		-\tfrac{a}{\bigLi}\pospart{-\pvi} 
	\end{equation*}
	for $a \geq 0$ and \(\bigLi < 0 < \bigUi\). 
	Then, 
	\begin{equation*}
		\conj{\funcsymbAppBLemmaTwo}(\cvi) 
		= 
		\pospart{\bigUi\cvi-a} 
		+ \pospart{\bigLi\cvi - a} 
		.
	\end{equation*}	
	This expression remains valid for \(\bigLi \leq 0 \leq \bigUi\) as long as we use the convention ``$0/0=0$''.
\end{lemma}

\newcommand{\bufcvi}{\beta}
\newcommand{\bufu}{\mu}
\begin{proof}
	Prior to proving the lemma, let us first show that the following result holds true: for all scalars \(\bufcvi\in\kR,\bufu\geq 0\), we have
	\begin{equation}
		\label{eq:proof:lemma:conj setnone:intermediary_result}
		\max_{0 \leq \pvi \leq \bufu} \; \bufcvi\pvi - \tfrac{a}{\bufu}\pvi = \pospart{\bufu\bufcvi - a}.
	\end{equation}
	Note that the maximum in~\eqref{eq:proof:lemma:conj setnone:intermediary_result} is attained since the optimization problem amounts to maximizing a linear function over a (nonempty) compact set.
	Now, if \(\bufu=0\) we have
	\begin{equation*}
		\max_{0 \leq \pvi \leq \bufu} \; \bufcvi\pvi - \tfrac{a}{\bufu}\pvi 
		= 0
	\end{equation*} 
	by using the convention ``$0/0=0$''. 
	Moreover, if \(\bufu>0\):	
	\begin{subequations}
		\begin{align*}
			\max_{0 \leq \pvi \leq \bufu} \; \bufcvi\pvi - \tfrac{a}{\bufu}\pvi
			\,=\,& \max_{0 \leq \pvi' \leq 1} \; \pvi'(\bufu\bufcvi - a)
			\\
			\,=\,& \pospart{\bufu\bufcvi - a}.
		\end{align*}
	\end{subequations}	 
	Overall, one can summarize the two cases as
	\begin{equation}
		\max_{0 \leq \pvi \leq \bufu} \; \bufcvi\pvi - \tfrac{a}{\bufu}\pvi
		=
		\pospart{\bufu\bufcvi - a}
	\end{equation}
	since \(a\geq 0\).

	\vspace*{.2cm}
	We now turn to the proof of \Cref{lemma:conj-setnone}. 
	By definition of the convex conjugate (see~\cite[Def.~4.1]{Beck2017aa}), one has \(\forall \cvi\in\kR\): 
	\begin{equation*}
		\conj{\funcsymbAppBLemmaTwo}(\cvi) = \sup_{\pvi\in\kR} \; \cvi\pvi - \funcsymbAppBLemmaTwo(\pvi).
	\end{equation*}
	The above problem can be equivalently expressed as
	\begin{equation}
		\label{eq:conj-split}
		\conj{\funcsymbAppBLemmaTwo}(\cvi) = \max\big\{\conj{\funcsymbAppBLemmaTwo}_1(\cvi) ; \conj{\funcsymbAppBLemmaTwo}_2(\cvi)\big\}
	\end{equation}
	where 
	\begin{subequations}
		\begin{align*}
			\conj{\funcsymbAppBLemmaTwo}_1(\cvi) &\triangleq \sup_{\pvi \geq 0} \; \cvi\pvi - \funcsymbAppBLemmaTwo(\pvi) \\
			\conj{\funcsymbAppBLemmaTwo}_2(\cvi) &\triangleq \sup_{\pvi \leq 0} \; \cvi\pvi - \funcsymbAppBLemmaTwo(\pvi).
		\end{align*}
	\end{subequations}
	On the one hand, \(\conj{\funcsymbAppBLemmaTwo}_1(\cvi)\) can be rewritten as 
	\begin{subequations}
		\begin{align*}
			\conj{\funcsymbAppBLemmaTwo}_1(\cvi) 
			\,=\,& \sup_{0 \leq \pvi \leq \bigUi} \; \cvi\pvi - \tfrac{a}{\bigUi}\pvi 
			.
		\end{align*}
	\end{subequations}
	By applying~\eqref{eq:proof:lemma:conj setnone:intermediary_result} with \((\bufcvi,\bufu) = (\cvi,\bigUi)\), we then obtain:
	\begin{equation} \label{eq:expression g1*}
		\conj{\funcsymbAppBLemmaTwo}_1(\cvi) = \pospart{\bigUi\cvi - a}
		.
	\end{equation}
	On the other hand, we have
	\begin{subequations}
		\begin{align*}
			\conj{\funcsymbAppBLemmaTwo}_2(\cvi) 
			\,=\,& \sup_{\bigLi \leq \pvi \leq 0} \; \cvi\pvi - \tfrac{a}{\bigLi}\pvi
			\nonumber \\
			\,=\,& \sup_{0 \leq \pvi' \leq -\bigLi} \; -\cvi\pvi' - \tfrac{a}{-\bigLi}\pvi' 
			.
		\end{align*}
	\end{subequations}
	Hence, using ~\eqref{eq:proof:lemma:conj setnone:intermediary_result} with \((\bufcvi, \bufu)=(\cvi,-\bigLi)\) leads to 
	\begin{equation} \label{eq:expression_g2_star}
		\conj{\funcsymbAppBLemmaTwo}_2(\cvi)
		=
		\pospart{\bigLi\cvi - a}
		.
	\end{equation}
	Finally, plugging \eqref{eq:expression g1*}-\eqref{eq:expression_g2_star} into \eqref{eq:conj-split} yields
	\begin{subequations}
		\begin{align*}
			\conj{\funcsymbAppBLemmaTwo}(\cvi) 
			&= \max\big\{
				\pospart{\bigUi\cvi - a}
				;
				\pospart{\bigLi\cvi - a}
				\big\} \\
			&= \pospart{\bigUi\cvi - a} + \pospart{\bigLi\cvi - a}
		\end{align*}
	\end{subequations}
	where the last equality follows from the fact that
	\begin{align*}
		\bigUi\cvi - a>0 &\implies \bigLi\cvi - a\leq 0
		.
	\end{align*}
	Indeed, if $\bigUi\cvi - a>0$ then necessarily $\bigUi>0$ since $a\geq 0$. We thus have $\cvi >\tfrac{a}{\bigUi}\geq 0$.
	This implies that $\bigLi\cvi - a\leq 0$ since $\bigLi\leq 0$ and $a\geq 0$.
\end{proof}


\section{Additional numerical results} 
\label{sec:additional numerical results}

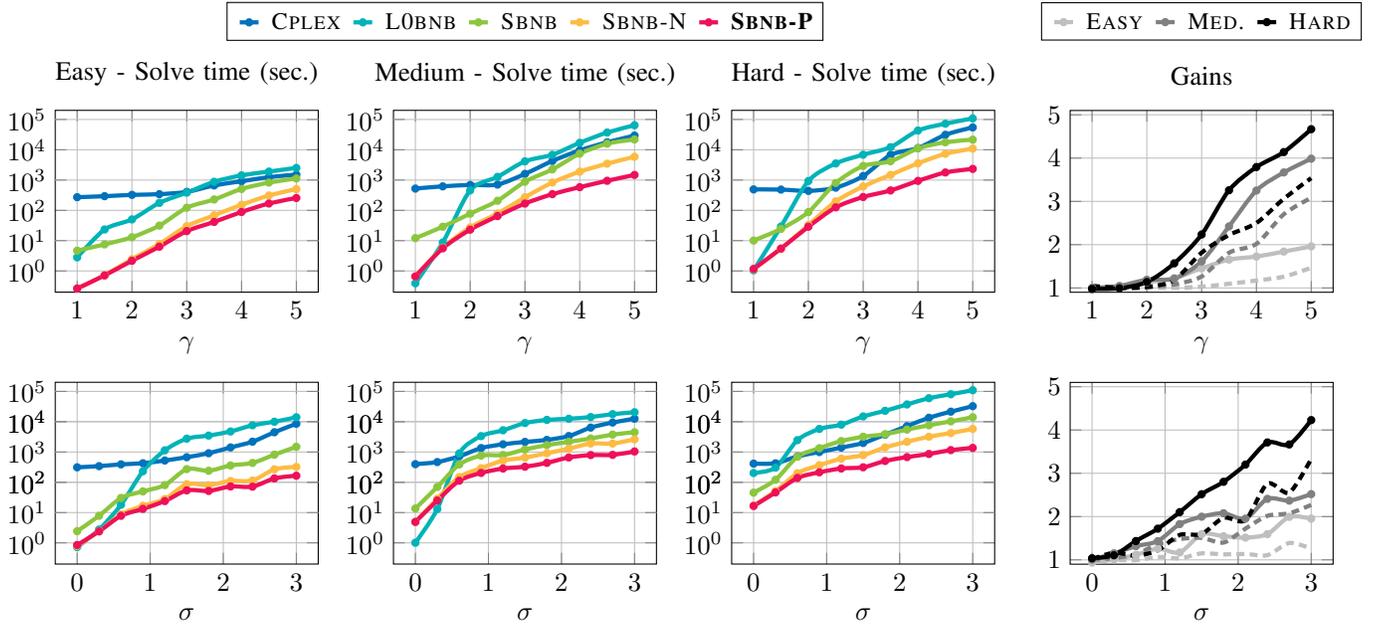
\begin{figure*}

\begin{tikzpicture}
	\begin{groupplot}[
		group style = {
			group size 			= 4 by 2,
			horizontal sep 		= 1cm,
			vertical sep 		= 1.2cm, 
		},
		height				= 4cm,
		width				= 0.28\textwidth,
		grid 				= both,
		cycle list name 	= mylist,
		legend cell align 	= left,
		legend pos 			= north west,
		legend style = {
			fill			= white, 
			fill opacity	= 0.6, 
			draw opacity 	= 1, 
			text opacity 	= 1,
			font 			= \scriptsize,
		},
		xtick distance		= 1,
	]

		\nextgroupplot[
			title 			= Easy - Solve time (sec.),
			ymode 			= log,
			ymin			= 2e-1,
			ymax			= 2e5,
			extra y ticks	= {10, 1000, 100000},
			xlabel 			= $\bigMfactor$,
		]
		\foreach \solver in {Cplex,L0bnb,Sbnb,SbnbNscr,SbnbPeel}{
			\addplot table [
				x		= bigmfactor,
				y		= \solver,
				col sep	= comma,
			] {dat/easy_synt_opt_linear_LeastSquares_Bigm_5_100_150_1.0_0.1_1000.0_x_solve_time.csv};
			\label{solver:\solver}
		}
		\coordinate (top) at (rel axis cs:0,1);

		\nextgroupplot[
			title 			= Medium - Solve time (sec.),
			ymode 			= log,
			ymin			= 2e-1,
			ymax			= 2e5,
			extra y ticks	= {10, 1000, 100000},
			xlabel 			= $\bigMfactor$,
		]
		\foreach \solver in {Cplex,L0bnb,Sbnb,SbnbNscr,SbnbPeel}{
			\addplot table [
				x		= bigmfactor,
				y		= \solver,
				col sep	= comma,
			] {dat/medium_synt_opt_linear_LeastSquares_Bigm_7_100_150_1.0_0.1_1000.0_x_solve_time.csv};
		}

		\nextgroupplot[
			title 			= Hard - Solve time (sec.),
			ymode 			= log,
			ymin			= 2e-1,
			ymax			= 2e5,
			extra y ticks	= {10, 1000, 100000},
			xlabel 			= $\bigMfactor$,
		]
		\foreach \solver in {Cplex,L0bnb,Sbnb,SbnbNscr,SbnbPeel}{
			\addplot table [
				x		= bigmfactor,
				y		= \solver,
				col sep	= comma,
			] {dat/hard_synt_opt_linear_LeastSquares_Bigm_7_100_150_1.0_0.8_1000.0_x_solve_time.csv};
		}

		\nextgroupplot[
			title = Gains,
			xlabel 	= $\bigMfactor$,
			ymin	= 0.9,
			ymax	= 5.1,
		]
		\addplot[smooth,ultra thick,lightgray,mark=*,mark size=0.75pt] table [
			x		= bigmfactor,
			y expr	= {\thisrow{SbnbNscr}/\thisrow{SbnbPeel})},
			col sep	= comma,
		] {dat/easy_synt_opt_linear_LeastSquares_Bigm_5_100_150_1.0_0.1_1000.0_x_solve_time.csv};
		\label{instance:easy}
		\addplot[smooth,ultra thick,densely dashed,lightgray] table [
			x		= bigmfactor,
			y expr	= {\thisrow{SbnbNscr}/\thisrow{SbnbPeel})},
			col sep	= comma,
		] {dat/easy_synt_opt_linear_LeastSquares_Bigm_5_100_150_1.0_0.1_1000.0_x_node_count.csv};
		\addplot[smooth,ultra thick,gray,mark=*,mark size=0.75pt] table [
			x		= bigmfactor,
			y expr	= {\thisrow{SbnbNscr}/\thisrow{SbnbPeel})},
			col sep	= comma,
		] {dat/medium_synt_opt_linear_LeastSquares_Bigm_7_100_150_1.0_0.1_1000.0_x_solve_time.csv};
		\label{instance:medium}
		\addplot[smooth,ultra thick,densely dashed,gray] table [
			x		= bigmfactor,
			y expr	= {\thisrow{SbnbNscr}/\thisrow{SbnbPeel})},
			col sep	= comma,
		] {dat/medium_synt_opt_linear_LeastSquares_Bigm_7_100_150_1.0_0.1_1000.0_x_node_count.csv};
		\addplot[smooth,ultra thick,black,mark=*,mark size=0.75pt] table [
			x		= bigmfactor,
			y expr	= {\thisrow{SbnbNscr}/\thisrow{SbnbPeel})},
			col sep	= comma,
		] {dat/hard_synt_opt_linear_LeastSquares_Bigm_7_100_150_1.0_0.8_1000.0_x_solve_time.csv};
		\label{instance:hard}
		\addplot[smooth,ultra thick,densely dashed,black] table [
			x		= bigmfactor,
			y expr	= {\thisrow{SbnbNscr}/\thisrow{SbnbPeel})},
			col sep	= comma,
		] {dat/hard_synt_opt_linear_LeastSquares_Bigm_7_100_150_1.0_0.8_1000.0_x_node_count.csv};

		\coordinate (top2) at (rel axis cs:0,1);


		\nextgroupplot[
			ymode 			= log,
			ymin			= 2e-1,
			ymax			= 2e5,
			extra y ticks	= {10, 100, 1000, 10000, 1e5},
			xlabel 			= $\stdnnz$,
		]
		\foreach \solver in {Cplex,L0bnb,Sbnb,SbnbNscr,SbnbPeel}{
			\addplot table [
				x		= sigma,
				y		= \solver,
				col sep	= comma,
			] {dat/easy_synt_opt_linear_LeastSquares_Bigm_5_100_150_x_0.1_1000.0_3.0_solve_time.csv};
		}

		\nextgroupplot[
			ymode 			= log,
			ymin			= 2e-1,
			ymax			= 2e5,
			extra y ticks	= {10, 100, 1000, 10000, 1e5},
			xlabel 			= $\stdnnz$,
		]
		\foreach \solver in {Cplex,L0bnb,Sbnb,SbnbNscr,SbnbPeel}{
			\addplot table [
				x		= sigma,
				y		= \solver,
				col sep	= comma,
			] {dat/medium_synt_opt_linear_LeastSquares_Bigm_7_100_150_x_0.1_1000.0_3.0_solve_time.csv};
		}

		\nextgroupplot[
			ymode 			= log,
			ymin			= 2e-1,
			ymax			= 2e5,
			extra y ticks	= {10, 100, 1000, 10000, 1e5},
			xlabel 			= $\stdnnz$,
		]
		\foreach \solver in {Cplex,L0bnb,Sbnb,SbnbNscr,SbnbPeel}{
			\addplot table [
				x		= sigma,
				y		= \solver,
				col sep	= comma,
			] {dat/hard_synt_opt_linear_LeastSquares_Bigm_7_100_150_x_0.8_1000.0_3.0_solve_time.csv};
		}

		\coordinate (bot) at (rel axis cs:1,0);

		\nextgroupplot[
			xlabel 	= $\stdnnz$,
			ymin	= 0.9,
			ymax	= 5.1,
		]

		\addplot[smooth,ultra thick,lightgray,mark=*,mark size=0.75pt] table [
			x		= sigma,
			y expr	= {\thisrow{SbnbNscr}/\thisrow{SbnbPeel})},
			col sep	= comma,
		] {dat/easy_synt_opt_linear_LeastSquares_Bigm_5_100_150_x_0.1_1000.0_3.0_solve_time.csv};
		\addplot[smooth,ultra thick,densely dashed,lightgray] table [
			x		= sigma,
			y expr	= {\thisrow{SbnbNscr}/\thisrow{SbnbPeel})},
			col sep	= comma,
		] {dat/easy_synt_opt_linear_LeastSquares_Bigm_5_100_150_x_0.1_1000.0_3.0_node_count.csv};
		\addplot[smooth,ultra thick,gray,mark=*,mark size=0.75pt] table [
			x		= sigma,
			y expr	= {\thisrow{SbnbNscr}/\thisrow{SbnbPeel})},
			col sep	= comma,
		] {dat/medium_synt_opt_linear_LeastSquares_Bigm_7_100_150_x_0.1_1000.0_3.0_solve_time.csv};
		\addplot[smooth,ultra thick,densely dashed,gray] table [
			x		= sigma,
			y expr	= {\thisrow{SbnbNscr}/\thisrow{SbnbPeel})},
			col sep	= comma,
		] {dat/medium_synt_opt_linear_LeastSquares_Bigm_7_100_150_x_0.1_1000.0_3.0_node_count.csv};
		\addplot[smooth,ultra thick,black,mark=*,mark size=0.75pt] table [
			x		= sigma,
			y expr	= {\thisrow{SbnbNscr}/\thisrow{SbnbPeel})},
			col sep	= comma,
		] {dat/hard_synt_opt_linear_LeastSquares_Bigm_7_100_150_x_0.8_1000.0_3.0_solve_time.csv};
		\addplot[smooth,ultra thick,densely dashed,black] table [
			x		= sigma,
			y expr	= {\thisrow{SbnbNscr}/\thisrow{SbnbPeel})},
			col sep	= comma,
		] {dat/hard_synt_opt_linear_LeastSquares_Bigm_7_100_150_x_0.8_1000.0_3.0_node_count.csv};

		\coordinate (bot2) at (rel axis cs:1,0);

	\end{groupplot}

	\path (top|-current bounding box.north)-- coordinate(legendpos) (bot|-current bounding box.north);
	\matrix[
		matrix of nodes,
		anchor=south,
		draw,
		inner sep=0.2em,
		every node/.style={anchor=base west, font=\small},
	] at([yshift=1ex]legendpos)
	{
		\ref{solver:Cplex} & \textsc{Cplex} & 
		\ref{solver:L0bnb} & \textsc{L0bnb} & 
		\ref{solver:Sbnb} & \textsc{Sbnb} & 
		\ref{solver:SbnbNscr} & \textsc{Sbnb-N} & 
		\ref{solver:SbnbPeel} & \textbf{\textsc{Sbnb-P}} \\
	};

	\path (top2|-current bounding box.north)-- coordinate(legendpos2) (bot2|-current bounding box.north);
	\matrix[
		matrix of nodes,
		anchor=south,
		draw,
		inner sep=0.2em,
		every node/.style={anchor=base west, font=\small},
	] at([yshift=-3.25ex]legendpos2)
	{
		\ref{instance:easy} & \textsc{Easy} & 
		\ref{instance:medium} & \textsc{Med.} & 
		\ref{instance:hard} & \textsc{Hard} \\
	};
\end{tikzpicture}
    \caption{
        Three first columns: Solving time as a function of $\bigMfactor$  (top, $\stdnnz=1$) and $\stdnnz$ (bottom, $\bigMfactor=3$). 
        Last column: gain in terms of solving time (solid) and number of nodes explored (dashed) 
        with respect to 
        \textsc{Sbnb-N} for the three different types of hardness.
    }
    \label{fig:additional}
\end{figure*}

In \Cref{sec:results}, we only vary the parameters that may directly influence the \textit{Big-M} constraint.
To give a broader overview of the performance allowed by our peeling methodology, we have selected three different types of instances of \eqref{prob:prob-base} with increasing hardness.
We still fix $\ddim$, $\pdim$ and $\mathrm{SNR}$ but we vary $\sparsitylevel$ and $\corrparam$ 
to control the instance inherent complexity.
On the one hand, larger $\sparsitylevel$ leads to problems with a more complex combinatorial nature.
On the other hand, increasing $\corrparam$ leads to more correlation between the columns of $\dic$, which brings the local minima of \eqref{prob:prob-base} closer to the global ones.
Both of these effects make the problem harder to solve.
Our three setups are constructed with the following combination of the parameters:
\begin{center}
    \begin{tabular}{cccccc}
        \toprule
        & $\sparsitylevel$ & $\ddim$ & $\pdim$ & $\corrparam$ & SNR \\
        \midrule
        Easy & $5$ & $100$ & $150$ & $0.1$ & 15 \\
        Medium & $7$ & $100$ & $150$ & $0.1$ & 15 \\
        Hard & $7$ & $100$ & $150$ & $0.8$ & 15 \\
        \bottomrule
    \end{tabular}
\end{center}
We note that the maximum solution time allowed for the solver is 30 hours, that is just above $10^5$ seconds.
For particularly hard instances, some solvers hit this time limit.
Since it only concerns a minority of the trials, we have decided not to remove them, but this may have slightly lower their mean solution time.

\Cref{fig:additional} shows the solution time of the different methods for the three different types of instances when varying $\bigMfactor$ and $\stdnnz$.
We roughly observe the same behaviors as in \Cref{fig:sensibility}. 
A notable remark is that the more difficult the instance, the larger the gains permitted by our peeling methodology.
Our empirical analysis is that on harder instances, the \gls{bnb} tree has to be explored more extensively.
Tightening the bounds at a given node will therefore allow a greater number of nodes to benefit from the strengthening of the relaxations.
The solving time will therefore be further reduced.
When varying $\bigMfactor$, the gains of the method implementing our peeling strategy (\textsc{Sbnb-P}) again its best concurrent (\textsc{Sbnb-N}) almost reach a factor $2$ on \textsc{Easy} instances. 
On \textsc{Medium} instances, it reaches a factor $4$ and on \textsc{Hard} instances, it reaches a factor $5$. 

}{}

\end{document}